\documentclass[twoside,11pt]{article}

\newcommand{\dnote}[1]{}
\newcommand{\pnote}[1]{}
\newcommand{\omittext}[1]{}

\usepackage{jmlr2e}

\usepackage{hyperref} 

\jmlrheading{v}{2012}{pp}{sub}{pub}{David P. Helmbold and Philip M. Long}

\ShortHeadings{Necessity of Irrelevant Variables}{Helmbold and Long}
\firstpageno{1}

\usepackage{graphicx} 
\usepackage{epstopdf}
\usepackage{epsfig} 
\usepackage{subfigure} 
\usepackage{amsmath,amssymb,amsfonts}
\usepackage{color,float,verbatim}
\newcommand{\qed}{\hfill \ensuremath{\Box}}
\newcommand{\PR}[1]{ \mspace{2mu} {\mathbb{P}}  \mspace{-2mu} \left[  #1 \right]}
\newcommand{\QR}[1]{ \mspace{2mu} {\mathbb{Q}}  \mspace{-2mu} \left[  #1 \right]}

\newcommand{\EXP}[1]{\exp \left( #1 \right) }
\newcommand{\bcrit}{ {\beta^*} }
\newcommand{\Rel}{ {\cal R} }

\newcommand{\E}{\mathbb{E}}

\newcommand{\eqdef}{\ensuremath{\mathrel{\stackrel{\mathrm{def}}{=}}}}

\newcommand{\gmax}{\gamma_{\text{\rm max}}}
\newcommand{\gmin}{\gamma_{\text{\rm min}}}

\newcommand{\bndref}[1]{Bound~(\ref{#1})}

{ \rm
\begin{tabbing}
....\=.....\=.....\=.....\=.....\=  \+ \kill
} %
{\end{tabbing} }

\newcommand{\note}[1]{} 
\newcommand{\M}{ {\ensuremath {\cal M} } }
\newcommand{\Mb}{ {\ensuremath {\cal M}_{\beta} } }
\newcommand{\Mcrit}{ {\ensuremath {\cal M}_{\bcrit} } }

\begin{document} 

\title{On the Necessity of Irrelevant Variables }


\author{\name David P.\ Helmbold \email dph@soe.ucsc.edu \\
       \addr Department of Computer Science\\
       University of California, Santa Cruz \\
       Santa Cruz, CA 95064, USA
       \AND
       \name Philip M.\ Long \email plong@sv.nec-labs.com \\
       \addr NEC Labs America \\
       10080 N. Wolfe Rd, SW3-350 \\
       Cupertino, CA 95014, USA}

\editor{G\'{a}bor Lugosi}

\maketitle

\begin{abstract}
This work explores the effects of relevant and
irrelevant boolean variables on the accuracy of classifiers.
The analysis uses the assumption that the variables are
conditionally independent given the class, and focuses on a
natural family of learning algorithms for such sources when the
relevant variables have a small advantage over random guessing.  
The main result is that algorithms relying predominately on irrelevant
variables have error probabilities that quickly go to $0$ in
situations where algorithms that limit the use of irrelevant variables
have errors bounded below by a positive constant.  We also show that
accurate learning is possible even when there are so few examples that
one cannot determine with high confidence whether or not any
individual variable is relevant.
\end{abstract} 

\begin{keywords}
Feature Selection, Generalization, Learning Theory
\end{keywords}

\section{Introduction}

When creating a classifier, a natural inclination is
to only use variables that are obviously relevant since irrelevant 
variables typically decrease the accuracy of a classifier. 
On the other hand,
this paper shows that 
the harm from irrelevant variables can be much less than the benefit from relevant variables
and therefore it is possible to learn very accurate classifiers even when almost all of the variables are irrelevant. 
It can be advantageous to continue adding variables, even as their prospects for being relevant fade away. Ê
We show this with theoretical analysis and experiments using artificially generated data. Ê

We provide an illustrative analysis that isolates the effects
of relevant and irrelevant variables on a classifier's accuracy. 
We analyze the case in which variables complement one another, which we
formalize using the common assumption of conditional independence
given the class label.  
We focus on the situation where relatively few of the many variables are relevant,  and the relevant variables are only weakly predictive.\footnote{Note that in many
natural settings the individual variables are only weakly associated
with the class label.  
This can happen when a lot of measurement error
is present, as is seen in microarray data. }
Under these conditions, algorithms that cast a wide net can succeed while more selective algorithms fail.

We prove upper bounds
on the error rate of a very simple learning algorithm that may 
include many irrelevant variables in its hypothesis.
We also prove a contrasting lower bound on the error
of every learning algorithm that uses mostly relevant
variables.
The combination of these results show that
the simple algorithm's 
error rate approaches zero 
in situations where every algorithm that predicts with mostly relevant variables 
has an error rate greater than a positive constant.

Over the past decade or so, a number of empirical and theoretical
findings have challenged the traditional rule of thumb
described by \citet{bishop06} as follows.
\begin{quote}
One rough heuristic that is sometimes advocated is that the number of data points should 
be no less than some multiple (say 5 or 10) of the number of adaptive parameters in the model.
\end{quote}
The Support Vector Machine
literature \citep[see][]{Vap98}  views algorithms that compute
apparently complicated functions of a given set of variables as
linear classifiers applied to an expanded, even infinite, set of
features.  
These empirically perform well on test data,
and theoretical accounts have been given for this.  
Boosting and Bagging
algorithms also generalize well, despite combining large numbers 
of simple classifiers -- even if the number of such ``base
classifiers'' is much more than the number of training examples
\citep{Qui96,Bre98,SFBL98}.  This is despite the fact that 
\citet{FHT00} showed
the behavior of such classifiers is closely related to
performing logistic regression on a potentially vast set of
features (one for each possible decision tree, for example).

Similar effects are sometimes found even when the features added are
restricted to the original ``raw'' variables.
Figure~\ref{f:shrunkencentroids}, which is reproduced from
\citep{THNC02}, is one example.  The curve labelled ``te'' is the
test-set error, and this error is plotted as a function of the number
of features selected by the Shrunken Centroids algorithm.  The best
accuracy is obtained using a classifier that depends on the expression
level of well over 1000 genes, despite the fact that there are only a
few dozen training examples.

\begin{figure}
\begin{center}
\includegraphics[width=4in]{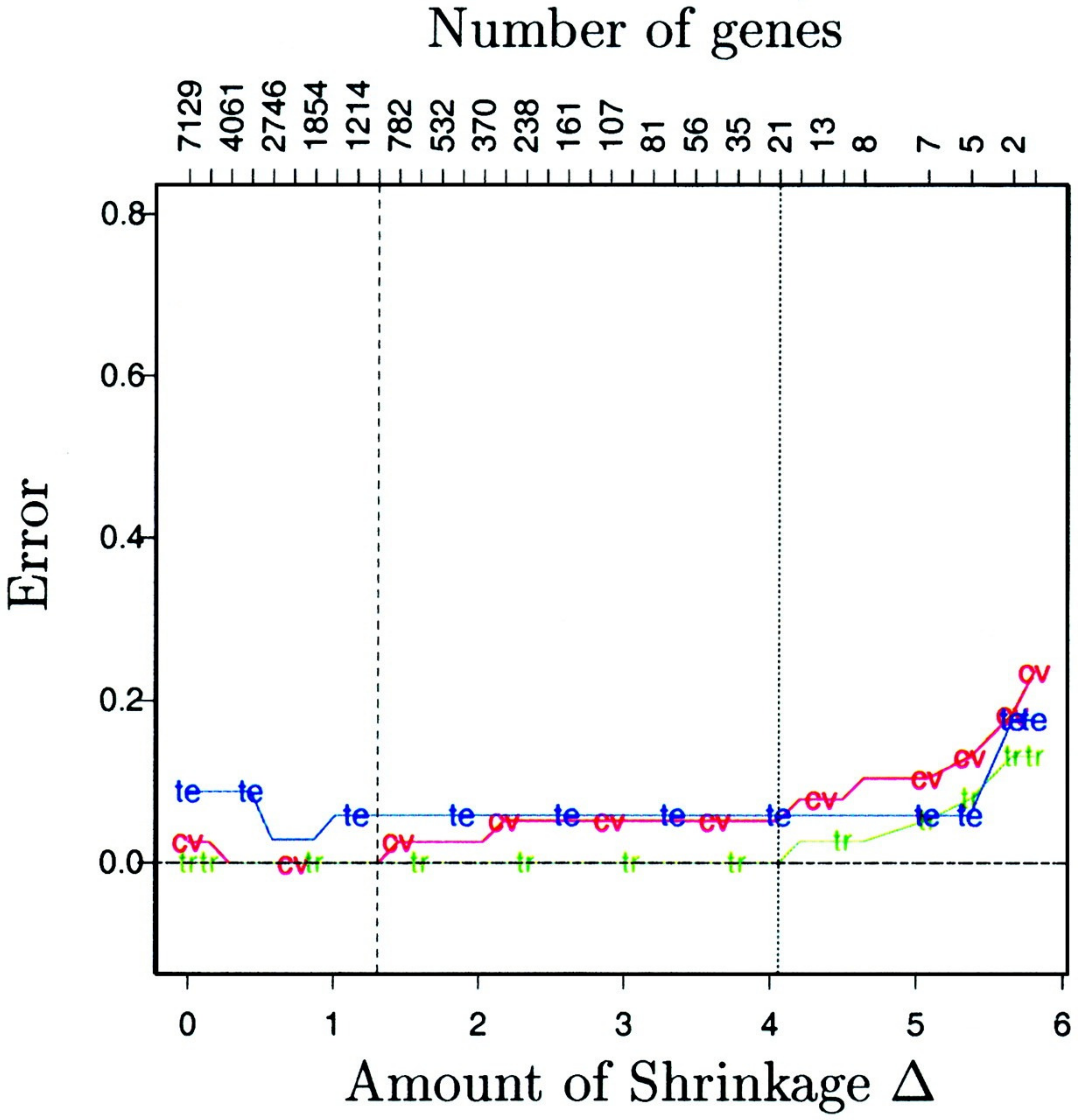} \\
\end{center}
\caption{
This graph is reproduced from \citep{THNC02}.  For a
  microarray dataset, the training error, test error, and
  cross-validation error are plotted as a function both of the number
  of features (along the top) included in a linear model and a
  regularization parameter $\Delta$ (along the bottom).}
\label{f:shrunkencentroids}
\end{figure}

It is impossible to tell if most of the variables used by the most
accurate classifier in Figure~\ref{f:shrunkencentroids} are
irrelevant.  However, we do know which variables are
relevant and irrelevant in synthetic data (and can generate as many 
test examples as desired).  
Consider for the moment a simple algorithm applied to a
simple source.  Each of two classes is equally likely, and there are
$1000$ relevant boolean variables, $500$ of which agree with the class label
with probability $1/2 + 1/10$, and $500$ which disagree with the class
label with probability $1/2 + 1/10$.  Another $99000$ boolean variables
are irrelevant.  The algorithm is equally simple: it has a parameter
$\beta$, and outputs the majority vote over those features (variables or
their negations) that agree with the class label on 
a $1/2 +\beta$ fraction of the training examples.  
Figure~\ref{f:synth} plots three runs of this
algorithm with $100$ training examples, and $1000$ test examples.
Both the accuracy of the classifier and the fraction of relevant
variables are plotted against the number of variables used in the
model, for various values of $\beta$.\footnote{ In the first graph,
only the results in which fewer than 1000 features were chosen 
are shown, since 
including larger feature sets obscures the shape of the graph in the
most interesting region, where relatively few features are chosen.}
Each time, the best accuracy is
achieved when an overwhelming majority of the variables used in the
model are irrelevant, and those models with few ($<25\%$) 
irrelevant variables perform far worse.
Furthermore, the best accuracy is obtained with
a model that uses many more variables than there are training
examples.  
Also, accuracy over 90\% is achieved even though there
are few training examples 
and the correlation of the individual variables with the class label is 
weak.
In fact, the number of examples is so small and the correlations are
so weak that, for any individual feature, it is impossible to
confidently tell whether or not  the
feature is relevant.

\begin{figure}[tb]
\begin{center}
\subfigure{\includegraphics[width=2.75in]{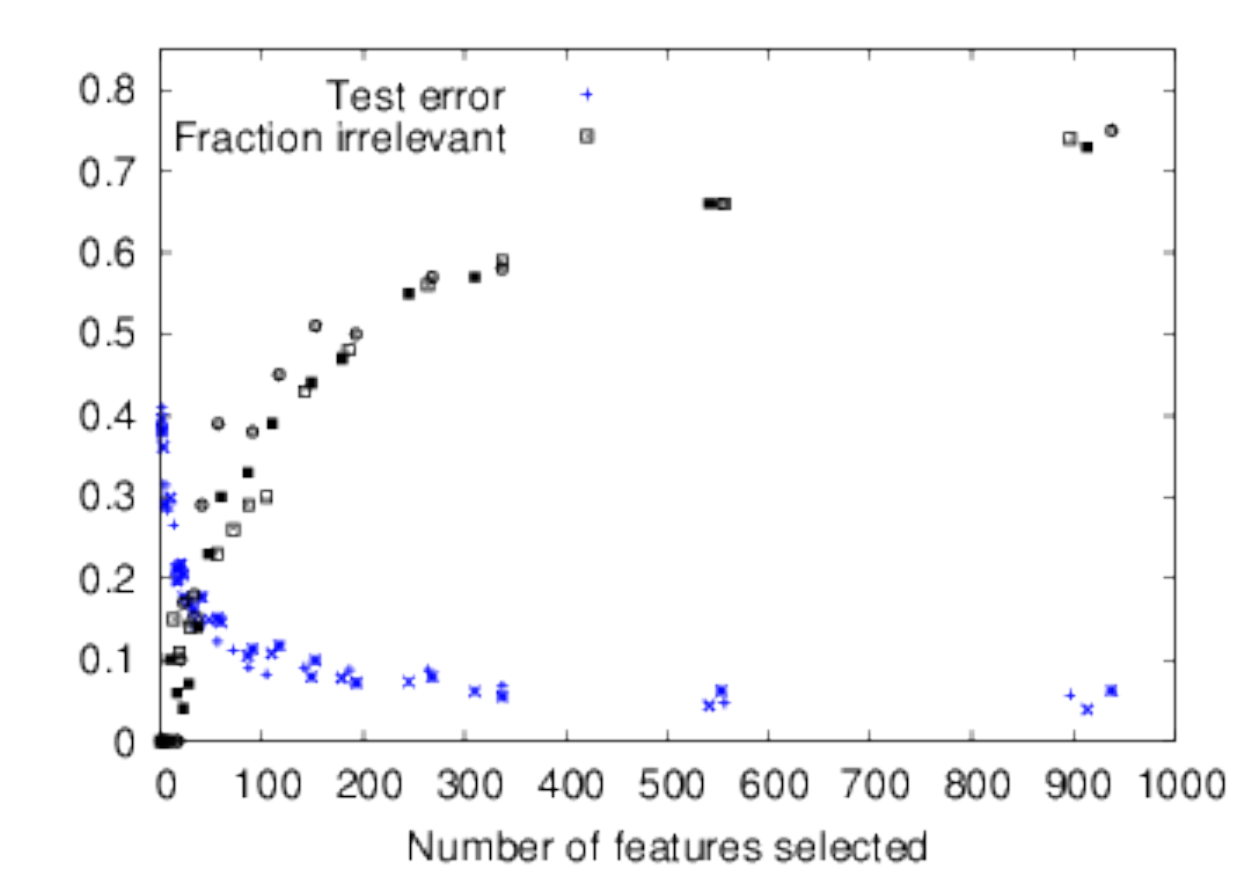}}
\subfigure{\includegraphics[width=2.75in]{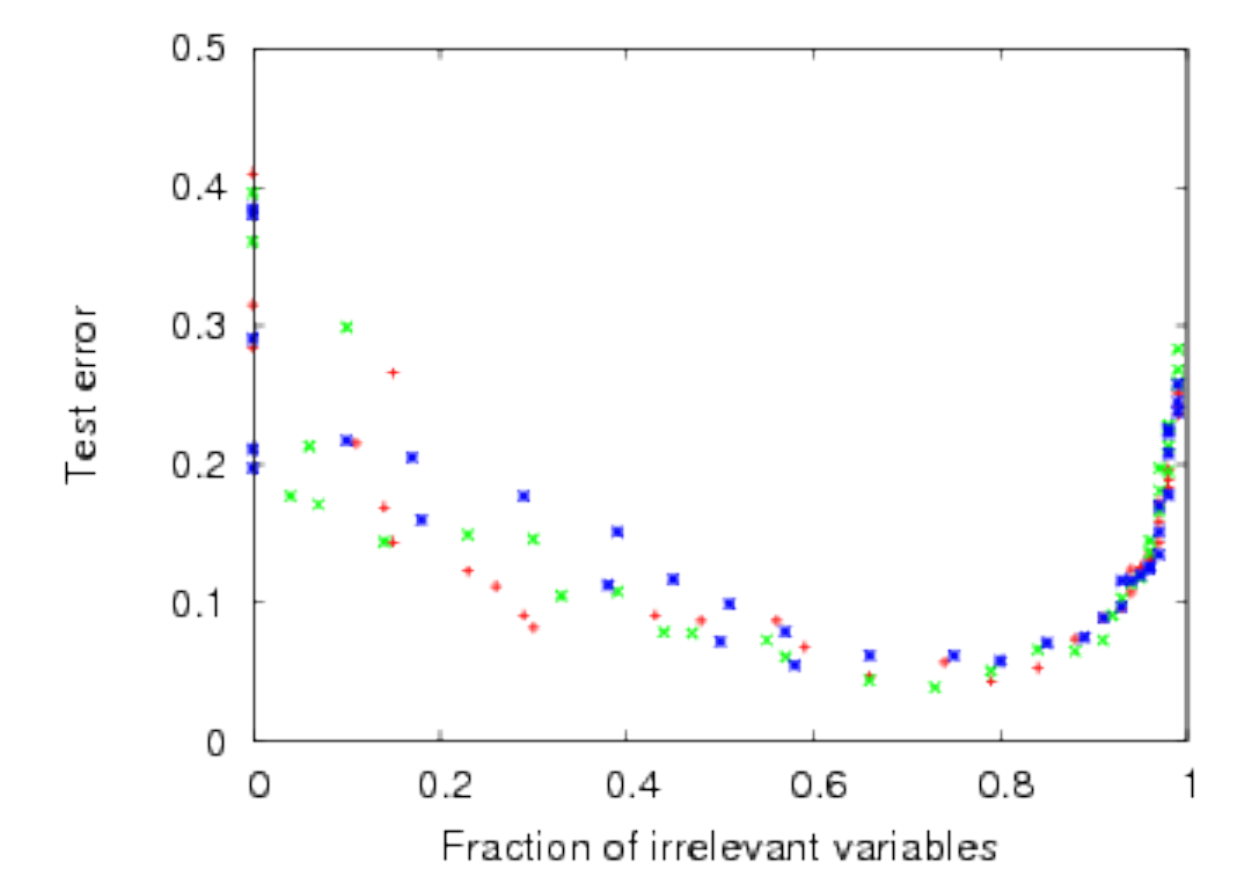}}
\end{center}
\caption{ \label{f:synth}
Left: Test error (blue) and fraction of irrelevant variables (black) 
as a function of the number of features. 
Right: Scatter plot of test error rates (vertical) against
fraction of irrelevant variables (horizontal).}
\end{figure}

Assume classifier $f$ consists of a vote over $n$ variables that are conditionally 
independent given the class label.
Let $k$ of the variables agree with the class label with probability $1/2 + \gamma$, and 
the remaining $n-k$ variables agree with the label with probability $1/2$.
Then the probability that $f$ is incorrect is at most 
\begin{equation}
\label{e:nolearn.intro}
\exp\left( \frac{-2 \gamma^2 k^2}{n} \right)
\end{equation}
(as shown in Section~\ref{s:basic}).  The error bound decreases exponentially
in the {\em square} of the number of relevant variables.  The
competing factor increases only {\em linearly} with the number of irrelevant
variables.  Thus, a very accurate classifier can be obtained with a
feature set consisting 
predominantly of irrelevant variables.

In Section~\ref{s:learning} we consider learning from training data where
the variables are
conditionally independent given the class label.
Whereas Equation~(\ref{e:nolearn.intro}) bounded the error as a
function of the number of variables $n$ and relevant variables $k$ in the
{\em model}, we now use capital 
$N$ and capital $K$ for the total number of  variables and
number of relevant variables in the {\em data}. 
The $N-K$ irrelevant variables are independent of the label,  agreeing with it with
probability $1/2$.  The $K$ relevant variables either agree with the
label with probability $1/2 + \gamma$ or with probability $1/2 - \gamma$.  
We analyze an algorithm
that chooses a value $\beta\geq 0$ and outputs a majority vote over all
features that agree with the class label on at least $1/2 + \beta$ of the
training examples
(as before, each feature is either a variable or its negation).  
Our Theorem~\ref{t:learning.beta} shows that if $\beta \leq \gamma$ and the
algorithm is given $m$ training examples, then the probability that it
makes an incorrect prediction on an independent test example is at
most
\begin{equation}
\label{e:upper.with.beta}
(1 + o(1)) \exp \left( - 2 \gamma^2  K \left(
        \frac{[1 - 8e^{-2 (\gamma - \beta)^2 m} - \gamma)]_+^2}
             {1 + 8 (N/K) e^{-2 \beta^2 m} + \gamma}
                  \right) \right),
\end{equation}
where $[z]_+ \eqdef \max \{ z, 0 \}$.  (Throughout the paper, the
``big Oh'' and other asymptotic notation will be for the case where
$\gamma$ is small, $K \gamma$ is large, and $N/K$ is large.  Thus the
edge of the relevant features and the fraction of features that are
relevant both approach zero while the total number of relevant
features increases.  If $K$ is not large relative to $1/\gamma^2$,
even the Bayes optimal classifier is not accurate.  No other
assumptions about the relationship between the parameters are needed.)

When $\beta \leq \gamma/2$ and the number $m$ of training examples
satisfies $m \geq c/\gamma^2$ for an absolute constant $c$, we also
show in Theorem~\ref{t:beta.cgamma} that the error probability is at
most
\begin{equation}
\label{e:inclusive.intro}
(1+o(1)) \exp \left( - \gamma^2 K^2/N \right).
\end{equation}
If $N = o(\gamma^2 K^2)$, this error probability goes to zero.
With only $\Theta(1/\gamma^2)$ examples, 
an algorithm cannot even tell with high confidence
whether a relevant variable is positively or negatively associated
with the class label, much less solve the more difficult problem of
determining whether or not a variable is relevant.
Indeed, this error bound is also achieved using
$\beta = 0$, when, for each variable $X_i$, the algorithm includes either $X_i$
or its negation in the vote.\footnote{\label{footnote}To be precise, the algorithm includes each variable or its negation
when $\beta = 0$ and $m$ is odd,  and includes both the
 variable and its negation when $m$ is even and the variable agrees
 with the class label exactly half the time.  But, any time both
a variable and its negation are included, their votes cancel.
We will always use the smaller equivalent model obtained by 
removing such canceling votes.
} 
Because bound~\eqref{e:inclusive.intro}
holds even when $\beta = 0$, it can be achieved by an algorithm that does not
use knowledge of $\gamma$ or $K$.

Our upper bounds illustrate the potential rewards for algorithms that
are ``inclusive'',  using many of the available variables in their
classifiers -- even when this means that most variables in the model
are irrelevant.  
We also prove a complementary lower bound that 
illustrates the potential cost when algorithms are ``exclusive''.  
We say that an algorithm is $\lambda$-exclusive if the expectation 
of the fraction of the variables used in its model that are relevant 
is at least $\lambda$.
We show that any $\lambda$-exclusive policy has
an error probability bounded below by 
$\lambda/4$ 
as $K$ and $N/K$ go
to infinity and $\gamma$ goes to $0$ in such a way that the error rate
obtained by the more ``inclusive'' setting $\beta = \gamma/2$ goes to
$0$.  In particular, no $\lambda$-exclusive algorithm (where
$\lambda$ is a positive constant) 
can achieve a bound like (\ref{e:inclusive.intro}).

\paragraph{Relationship to Previous Work}

Donoho and Jin \citep[see][]{DJ08,Jin09} and 
Fan and Fan \citeyearpar{FF08},
building on a line of research on joint testing of multiple
hypotheses \citep[see][]{ABDJ06,ABDG10,DJ04,DJ06,MR06},
performed analyses and simulations using sources with elements in
common with the model studied here, including conditionally
independent variables and a weak association between the variables
and the class labels.
Donoho and Jin 
also pointed out that their
algorithm can produce accurate hypotheses while using many more
irrelevant features than relevant ones.  
The main theoretical results proved in their
papers
describe conditions that imply
that, if the relevant variables are too small a fraction of all the
variables, and the number of examples is too small, then learning is
impossible.  The emphasis of our theoretical analysis is the opposite:
algorithms can tolerate a large number of irrelevant variables,
while using a small number of examples, and algorithms that
avoid irrelevant variables, even to a limited extent, cannot learn
as effectively as algorithms that cast a wider net.
In particular, ours is the first analysis
that we are aware of to have a result qualitatively like
Theorem~\ref{t:lower.lambda}, which demonstrates the limitations of
exclusive algorithms.  

For
the sources studied in this paper, there is a linear
classifier that classifies most random examples correctly with a large
margin, i.e.\ most examples are not close to the decision boundary.
The main motivation for our analysis was to understand the effects of
relevant and irrelevant variables on generalization, but it is
interesting to note that we get meaningful bounds in the extreme case
that $m = \Theta(1/\gamma^2)$, whereas the margin-based bounds that we
are aware of (such as \citet{SFBL98,KP02,DL03,WSetal08}) are vacuous in this
case.  (Since these other bounds hold 
more generally, their overall strength is incomparable to our results.)
\citet{NJ01} showed that the
Naive Bayes algorithm (which ignores class-conditional dependencies)
converges relatively quickly, justifying its use when there are few
examples.  But their bound for Naive Bayes is also vacuous when $m =
\Theta(1/\gamma^2)$.  \citet{BL04} studied the case
in which the class conditional distributions are Gaussians, and
showed how an algorithm which does not model class conditional
dependencies can perform nearly optimally in this case, especially
when the number of variables is large. B\"uhlmann and Yu 
\citeyearpar{BY02}
analyzed the variance-reduction benefits of Bagging with primary
focus on the benefits of the smoother classifier that is obtained
when ragged classifiers are averaged.  As such it takes a different
form than our analysis.

Our analysis demonstrates that certain effects are possible, but how
important this is depends on how closely natural learning settings
resemble our theoretical setting and the extent to which our analysis
can be generalized.  The conditional independence assumption is one
way to express the intuitive notion that variables are not too
redundant.  A limit on the redundancy is needed for results like ours
since, for example, a collection of $\Theta(k)$ perfectly correlated
irrelevant variables would swamp the votes of the $k$ relevant
variables.  On the other hand, many boosting algorithms minimize the
potential for this kind of effect by choosing features in later
iterations that make errors on different examples then the previously
chosen features.  One relaxation of the conditional independence
assumption is to allow each variable to conditionally depend on a
limited number $r$ of other variables, as is done in the formulation
of the Lovasz Local Lemma \citep[see][]{ASE92}.  
As partial illustration of the
robustness of the effects analyzed here, we
generalize upper bound~(\ref{e:nolearn.intro}) to this case in
Section~\ref{s:dependent}.  
There we prove an error bound of 
$c (r+1) \exp\left(\frac{-2\gamma^2 k^2}{n (r + 1)}\right)$ 
when each variable
depends on most $r$ others.  There are a number of ways that one could
imagine relaxing the conditional independence assumption while still
proving theorems of a similar flavor.  Another obvious direction for
generalization is to relax the strict categorization of variables into
irrelevant and $(1/2+\gamma)$-relevant classes.  We believe that many
extensions of this work with different coverage and interpretability
tradeoffs are possible.  For example, our proof techniques easily give
similar theorems when each relevant variable has a probability between
$1/2 + \gamma/2$ and $1/2 + 2 \gamma$ of agreeing with the class label
(as discussed in Section~\ref{s:different}).
Most of this paper uses the cleanest and simplest setting in order to focus
attention on the main ideas. 

We state some useful tail bounds in the next section, and
Section~\ref{s:basic} analyzes the error of simple voting classifiers.
Section~\ref{s:learning} gives bounds on the expected error of
hypotheses learned from training data while Section~\ref{s:lower} shows that,
in certain situations, any exclusive algorithm must have high error
while the error of some inclusive algorithms goes to 0.  In
Section~\ref{s:dependent} we bound the accuracy of voting classifiers
under a weakened independence assumption and in
Section~\ref{s:different} we consider relaxation of the assumption
that all relevant variables have the same edge. 

\section{Tail bounds}
\label{s:tools}

This section gathers together the several tail bounds 
that will be used in various places in the
analysis.  These bounds all assume that $U_1, U_2, \ldots, U_{\ell}$
are $\ell$ independent $\{ 0,1 \}$-valued random variables and
$U=\sum_{i=1}^\ell U_i$.  We start with some upper bounds.

\begin{itemize}
\item
The Hoeffding bound, \citep[see][]{Pol84}:  
\begin{equation}
\label{e:hoeffding}
\PR{\frac{1}{\ell} U 
                - \E\left( \frac{1}{\ell} U \right)
         \geq \eta }
  \leq e^{-2 \eta^2 \ell}.
\end{equation}

\item
The Chernoff bound, \citep[see]{AV79,MR95} and
Appendix~\ref{a:chernoff}.  For any $\eta>0$:
\begin{align}
\label{e:chernoff}
\PR{U>(1+\eta)\E(U)} < \exp\left(-(1+\eta)\E(U) \ln \left(\frac{1+\eta}{e}\right) \right).
\end{align}

\item
For any $0\leq\eta \leq 4$ (see Appendix~\ref{a:chernoff}):  
\begin{equation}
\label{e:leq4}
\PR{U>(1+\eta)\E(U)}
  < \exp\left(-\eta^2 \E(U)/4 \right).
\end{equation}

\item
For any $0<\delta \leq 1$ (see Appendix~\ref{a:highconf}): 
\begin{equation}
\label{e:highconf}
\PR{U> 4\E(U) + 3 \ln(1/\delta)} < \delta .
\end{equation}
\end{itemize}

\noindent We also use the following  lower bounds on the tails of  distributions.

\begin{itemize}
\item If $\PR{U_i = 1} = 1/2$ for all $i$, $\eta > 0$, and $\ell \geq 1/\eta^2$
then (see  Appendix~\ref{a:lower.tail}):
\begin{equation}
\label{e:lower.tail}
\PR{\frac{1}{\ell}U - \frac{1}{\ell} \E\left( U \right)
         \geq \eta }
  \geq \frac{1}{7 \eta \sqrt{\ell}} 
         \exp\left(-2 \eta^2 \ell \right)
            - \frac{1}{\sqrt{\ell}}.
\end{equation}

\item
If $\PR{U_i=1} = 1/2$ for all $i$,  then for all $0 \leq \eta \leq 1/8$
such that $\eta \ell$ is an integer\footnote{For notational simplicity we omit the floors/ceilings implicit in the use of this bound.}
(see Appendix~\ref{a:lower.fair}): 
\begin{equation}
\label{e:lower.fair}
\PR{ \frac{1}{\ell}U - \frac{1}{\ell}\E(U) \geq \eta } \geq \frac{1}{5} e^{-16 \eta^2 \ell}.
\end{equation}

\item 
A consequence of Slud's Inequality~\citeyearpar{Slu77} gives the following
(see Appendix~\ref{a:unfair}).  If $0 \leq \eta \leq 1/5$ and 
$\PR{ U_i = 1 } = 1/2 + \eta$ for all $i$ then:
\begin{equation}
\label{e:unfair}
\PR{ \frac{1}{\ell} U < 1/2 } \geq \frac{1}{4} e^{-5 \eta^2 \ell}.
\end{equation}

\end{itemize}

\noindent
Note that the constants in the above bounds were chosen to be simple and illustrative, rather than the best possible.

\section{The accuracy of models containing relevant and irrelevant variables}
\label{s:basic}

In this section we analyze the accuracy of the models (hypotheses) 
produced by the algorithms in Section~\ref{s:learning}.
Each example is represented by a vector of $N$ binary \emph{variables}
and a class designation.
We use the following generative model:
\begin{itemize}
\item a random class designation from $\{0,1\}$ is chosen, with both
classes equally likely, then
\item each of $K$ \emph{relevant} variables are equal to the class designation
with probability $1/2 + \gamma$ (or with probability $1/2 - \gamma$), and
\item the remaining $N-K$ \emph{irrelevant} variables are equal to the 
class label
with probability $1/2$; 
\item all variables are conditionally independent given the class designation.
\end{itemize}

Which variables are relevant and whether each one is positively or
negatively correlated with the class designations are chosen
arbitrarily ahead of time.

A \emph{feature} is either a variable or its complement.  The $2(N-K)$
\emph{irrelevant} features come from the irrelevant variables, the $K$
\emph{relevant} features agree with the class labels with probability
$1/2+\gamma$, and the $K$ \emph{misleading} features agree
with the class labels with probability $1/2-\gamma$.

We now consider models $\M$ predicting with a majority vote over a subset
of the features.  We use $n$ for the total number of features in model
$\M$, $k$ for the number of relevant features, and $\ell$ for the
number of misleading features (leaving $n-k-\ell$ irrelevant
features).  Since the votes of a variable and its negation ``cancel
out,'' we assume without loss of generality that models include at
most one feature for each variable.  Recall that $[ z ]_{+} \eqdef
\max \{z,0\}$.

\begin{theorem}
\label{t:no.learning.misleading}
Let $\M$ be a majority vote of $n$ features, $k$ of which are
relevant and $\ell$ of which are misleading (and $n-k-\ell$ are
irrelevant).  The probability that $\M$ predicts
incorrectly is at most 
$\displaystyle \exp\left(\frac{-2 \gamma^2 [k  - \ell]_+^2 }{n} \right)$.
\end{theorem}

{\bf Proof}: If $\ell \geq k$ then the exponent is $0$ and the bound
trivially holds.  

Suppose $k > \ell$.  Model $\M$ predicts incorrectly only when at most
half of its features are correct.  The expected fraction of correct
voters is $1/2 + \frac{ \gamma (k - \ell) }{n}$, so, for $\M$'s
prediction to be incorrect, the fraction of correct voters must be at
least $\gamma (k - \ell) / n$ less than its expectation.
Applying~(\ref{e:hoeffding}), this probability is at most
\[
\exp\left(\frac{-2 \gamma^2 (k  - \ell)^2 }{n} \right).
\]
\qed

The next corollary shows that even models where most of the features
are irrelevant can be highly accurate.
\begin{corollary}
\label{c:mostly.irrelevant}
If $\gamma$ is a constant, $k -\ell= \omega(\sqrt{n})$ and $k = o(n)$, then
the accuracy of the model approaches $100\%$ while its fraction of
irrelevant variables approaches~$1$ (as $n\rightarrow \infty$).
\end{corollary}

For example, the conditions of Corollary~\ref{c:mostly.irrelevant} are
satisfied when $\gamma = 1/4$, $k = 2 n^{2/3}$ and $\ell = n^{2/3}$.

\section{Learning}
\label{s:learning}

We now consider the problem of learning a model $\M$ from data.  
We assume that the algorithm receives $m$
i.i.d.\ examples generated as described in Section~\ref{s:basic}.
One test example is independently generated from the same distribution,
and we evaluate the algorithm's \emph{expected error}: the probability
over training set and test example that its model makes an incorrect
prediction on the test example (the ``prediction model'' of
\citet{HLW94}).

We define $\M_{\beta}$ to be the majority vote\footnote{If $\Mb$ is
  empty or the vote is tied then any default prediction, such as 1,
  will do.}  of all features that equal the class label on at least
$1/2 + \beta$ of the training examples.  
To keep the analysis as clean
as possible, our results in this section apply to algorithms that chose $\beta$ as a
function of the number of features $N$, the number of relevant features $K$, the edge
of the relevant features $\gamma$, and training set size $m$, and then
predict with $\M_\beta$.
Note that this includes the algorithm that always choses $\beta=0$ regardless of $N$, $K$, $\gamma$ and $m$.

Recall that asymptotic notation will concern the case in which
$\gamma$ is small, $K \gamma$ is large, and $N/K$ is large.

This section proves two theorems bounding the expected error rates of
learned models.  One can compare these bounds with a similar bound on the Bayes Optimal predictor that
``knows'' which features are relevant. 
This Bayes Optimal
predictor for our generative model is a majority vote of the $K$
relevant features, and has an error rate bounded by $e^{-2\gamma^2 K}$
(a bound as tight as the Hoeffding bound).  

\begin{theorem}
\label{t:learning.beta}
If $0\leq \beta \leq \gamma$, then the expected error rate of $\Mb$ is at most
\[
(1 + o(1)) \exp \left( - 2 \gamma^2  K \left(
        \frac{[1 - 8e^{-2 (\gamma - \beta)^2 m} - \gamma ]_+^2}
             {1 + 8 (N/K) e^{-2 \beta^2 m} + \gamma }
                  \right) \right).
\]
\end{theorem}

Our proof of Theorem~\ref{t:learning.beta} starts with lemmas bounding
the number of misleading, irrelevant, and relevant features in $\Mb$.
These lemmas use a quantity $\delta > 0$ that will be determined later
in the analysis.

\begin{lemma}
\label{l:few.misleading}
With probability at least $1 - \delta$, the
number of misleading features in $\Mb$ is at most 
$4 K e^{-2 (\gamma + \beta)^2 m} + 3\ln (1/\delta).$
\end{lemma}
{\bf Proof}: For a particular misleading feature to be included in
$\Mb$, Algorithm $A$ must overestimate the probability that misleading feature equals the class label by
at least $\beta + \gamma$.  
Applying~(\ref{e:hoeffding}), this happens
with probability at most $e^{-2 (\beta+\gamma)^2 m}$, so the expected
number of misleading features in $\Mb$ is at most $K e^{-2
  (\beta+\gamma)^2 m}$.  Since each misleading feature is associated
with a different independent variable, we can apply~(\ref{e:highconf})
with $\E(U) \leq K e^{-2 (\beta+\gamma)^2 m}$ to get the desired result.
\qed

\begin{lemma}
\label{l:few.irrelevant}
With probability at least $1 - 2\delta$, the number of irrelevant features in
$\Mb$ is at most
$
8 N e^{-2 \beta^2 m} + 6 \ln (1/\delta).
$
\end{lemma}
{\bf Proof}: 
For a particular positive irrelevant feature to be
included in $\Mb$, Algorithm $A$ must overestimate the probability
that the positive irrelevant feature equals the class label by $\beta$.  
Applying~(\ref{e:hoeffding}), this
happens with probability at most $e^{-2 \beta^2 m}$, so the expected
number of irrelevant positive features in $\Mb$ is at most 
$(N -K) e^{-2 \beta^2 m}$.

All of the events that variables agree with the label, for various
variables, and various examples, are independent.  So the events that
various irrelevant variables are included in $\Mb$ are independent.
Applying~(\ref{e:highconf}) with $\E(U) = (N-K)e^{-2\beta^2 m}$ gives
that, with probability at least $1-\delta$, the number of irrelevant
positive features in $\Mb$ is at most $4(N-K) e^{-2 \beta^2 m}$.

A symmetric analysis establishes the same bound on the number of negative irrelevant
features in $\Mb$.   
Adding these up 
completes the proof. 
\qed

\begin{lemma}
\label{l:many.relevant}
With probability at least $1 - \delta$, the number of
relevant features in $\Mb$ is at least 
$
K - 4K e^{-2 (\gamma - \beta)^2 m} - 3\ln(1/\delta).
$
\end{lemma}
{\bf Proof}: 
For a particular relevant feature to be excluded
from $\Mb$, Algorithm $A$ must underestimate the probability that
the relevant feature equals the class label by  at least $\gamma - \beta$.  
Applying~(\ref{e:hoeffding}), this
happens with probability at most $e^{-2 (\gamma - \beta)^2 m}$, so the
expected number of relevant variables excluded from $\Mb$ is at most
$K e^{-2 (\gamma - \beta)^2 m}$.  Applying~(\ref{e:highconf}) as in
the preceding two lemmas completes the proof. 
\qed

\medskip

\begin{lemma} 
\label{l:error.rate}
The probability that $\Mb$ makes an error is at most
\[
\exp\left( \frac{-2 \gamma^2 \left[K- 8 Ke^{-2(\gamma-\beta)^2 m}
                                    -6  \ln (1/\delta)\right]_+^2}
		{K + 8N e^{-2\beta^2m} + 6\ln (1/\delta) }
           \right)
   + 4 \delta.
\]
for any $\delta > 0$ and $0 \leq \beta \leq \gamma$.
\end{lemma}
{\bf Proof}: The bounds of Lemmas~\ref{l:few.misleading},  \ref{l:few.irrelevant}, and 
\ref{l:many.relevant} simultaneously hold
with probability at least $1- 4 \delta$.  Thus the error probability
of $\Mb$ is at most $4 \delta$ plus the probability of error given
that all three bounds hold.  Plugging the three bounds into
Theorem~\ref{t:no.learning.misleading}, (and over-estimating 
the number $n$ of variables in the model with
$K$ plus the bound of Lemma~\ref{l:few.irrelevant} on the number
of irrelevant variables) gives a bound on
$\Mb$'s error probability of
\begin{equation}  \label{e:condErrBound}
\exp\left( 
	\frac{-2 \gamma^2 \left[ (K - 4K e^{-2 (\gamma - \beta)^2 m} - 3\ln(1/\delta))    - (4 K e^{-2 (\gamma + \beta)^2 m} + 3\ln (1/\delta))  \right]_{+}^2 } 
		{ K+ 8 N e^{-2 \beta^2 m} + 6 \ln (1/\delta)}  
	\right) 
\end{equation}
when all three bounds hold.  
Under-approximating $(\gamma + \beta)^2$ with $(\gamma-\beta)^2$ and simplifying yields:
\begin{align*}
\text{(\ref{e:condErrBound})} 
& \leq \exp\left( 
	\frac{-2 \gamma^2 \left[ K - 8 K e^{-2 (\gamma - \beta)^2 m} - 6 \ln(1/\delta) \right]_{+}^2 } 
		{ K+ 8 N e^{-2 \beta^2 m} + 6 \ln (1/\delta)}  
	\right) .
\end{align*}
Adding $4\delta$ completes the proof.
\qed

\medskip

We are now ready to prove Theorem~\ref{t:learning.beta}.

\medskip

{\bf Proof} (of Theorem~\ref{t:learning.beta}):
Using 
\[
\delta = \exp\left(- \frac{\gamma K}{6}  \right)
\]
 in Lemma~\ref{l:error.rate} bounds the probability that $\Mb$ makes a mistake by
 
 \begin{multline}  \label{e:twoTerms}
\exp  \left( 
	\frac{-2 \gamma^2 \left[ K - 8 K e^{-2 (\gamma - \beta)^2 m} - \gamma K \right]_{+}^2 } 
		{ K+ 8 N e^{-2 \beta^2 m} + \gamma K}  
	\right) + 4  \exp\left(- \frac{\gamma K}{6} \right) \\
<  \exp \left( 
	\frac{-2 \gamma^2  K \left[ 1 -  8 e^{-2 (\gamma - \beta)^2 m} - \gamma \right]_{+}^2 } 
		{ 1+ \frac{8 N}{K}  e^{-2 \beta^2 m} + \gamma }  
	\right) + 4  \exp\left(- \frac{\gamma K}{6} \right) .
 \end{multline}
 
The first term is at least $e^{- 2 \gamma^2 K}$, and
\[
 4 \exp\left(- \frac{\gamma K}{6} \right) =  o \left( e^{-2\gamma^2 K} \right)
 \]
as $\gamma \rightarrow 0$ and $\gamma K \rightarrow \infty$, so~(\ref{e:twoTerms}) implies the bound
\[
(1 + o(1)) 
  \exp \left( \frac{ - 2 \gamma^2  K \left[1 - 8 e^{-2 (\gamma - \beta)^2 m} 
                                           -  \gamma \right]_+^2}
                   {1+ \frac{8 N}{K} e^{-2 \beta^2 m} +  \gamma} \right)
\]
as desired.
\qed

The following theorem bounds the error in terms of just $K$, $N$, and $\gamma$
when $m$ is sufficiently large.

\begin{theorem}
\label{t:beta.cgamma}
Suppose algorithm $A$ produces models $\Mb$ where  $0\leq \beta \leq c \gamma$
for a constant $c \in [0,1)$.
\begin{itemize}
\item
Then there is a constant $b$ 
(depending only on $c$) 
such that whenever $m \geq b / \gamma^2$ the
error of $A$'s model is at most 
$(1 + o(1)) \exp\left( - \frac{\gamma^2 K^2}{N} \right) $.  
\item
If $m =\omega( 1 / \gamma^2)$ then the error of $A$'s model 
is at most $(1+o(1))\exp\left(\frac{-(2-o(1)) \gamma^2 K^2 }{N} \right)$.
\end{itemize}
\end{theorem}
\begin{proof}
Combining Lemmas~\ref{l:few.misleading} and~\ref{l:many.relevant} with
the upper bound of $N$ on the number of features in $\Mb$ as in
Lemma~\ref{l:error.rate}'s proof gives the following error bound on
$\Mb$
\[
\exp\left( 
	\frac{-2 \gamma^2 \left[ K -  8K e^{-2 (\gamma - \beta)^2 m} - 6 \ln(1/\delta) \right]_{+}^2 }{N}  
	\right) 
+  2 \delta	
\]
for any $\delta  > 0$.
Setting 
\[
\delta = \exp\left(- \frac{\gamma K}{6}  \right)
\]
and continuing as in the proof of Theorem~\ref{t:learning.beta}
gives the bound
\begin{equation}
\label{e:launching.pad}
(1 + o(1)) \exp\left( 
	\frac{-2 \gamma^2 K^2 \left[ 1 - 2 \left(4 e^{-2 (\gamma - \beta)^2 m} + \gamma  \right) \right]_{+}^2 } 
		{ N}  
	\right).
\end{equation}
For the first part of the theorem, it suffices to show that 
the $[\cdots]_+^2$ term is at least $1/2$.
Recalling that our analysis is for small $\gamma$,
the term inside the $[\cdots]_+$ of (\ref{e:launching.pad}) is at least 
\[
1- 8 e^{ -2 (1-c)^2 \gamma^2 m} - o(1) . 
\]
When
\begin{equation}  \label{e:mThreshold}
m \geq \frac{\ln \left( 32 \right) }{ 2 (1-c)^2 \gamma^2 }, 
\end{equation}
this term is at least $3/4 - o(1)$, and thus its square is at least
1/2 for small enough $\gamma$, completing the proof of the first part
of the theorem.

To see the second part of the theorem, since $m \in \omega (1 /
\gamma^2)$, the term of (\ref{e:launching.pad}) inside the
$[\cdots]_+$ is $1- o(1)$.
\end{proof}

By examining inequality~(\ref{e:mThreshold}), we see that the constant $b$ in Theorem~\ref{t:beta.cgamma} can be set
to $\ln(32) / 2 (1-c)^2$.

\begin{lemma}
\label{l:many.irrel}
The expected number of irrelevant variables in $\Mb$
is at least $(N-K) e^{-16 \beta^2 m}$.
\end{lemma}
\begin{proof}
Follows from inequality~(\ref{e:lower.fair}).
\end{proof}

\begin{corollary}
\label{c:inclusive.good}
If \ $K$, $N$, and $m$ are functions of $\gamma$ such that 
\begin{align*}
& \gamma \rightarrow 0,  \\
& K^2 / N \in \omega(\ln (1/ \gamma)/\gamma^2), \\
& K = o(N) \text{ and } \\
& m = 2 \ln(32) / \gamma^2  
\end{align*}
then if an algorithm outputs ${\cal M}_{\beta}$
using a $\beta$ in $[0, \gamma/2]$, it
has an error that decreases super-polynomially (in $\gamma$),
while the expected fraction of irrelevant variables in the model goes
to $1$.
\end{corollary}

Note that Theorem~\ref{t:beta.cgamma} and Corollary~\ref{c:inclusive.good} include non-trivial error bounds on the model
$\M_0$ that votes all $N$ variables (for odd sample size $m$).

\section{Lower bound}
\label{s:lower}

Here we show that any algorithm with an error guarantee
like Theorem~\ref{t:beta.cgamma} must include many irrelevant features
in its model.  
The preliminary version of this paper \citep{HL11} contains
a related lower bound  for algorithms that choose $\beta$ as a function of $N$, $K$, $m$, and $\gamma$, and
predict with ${\cal M}_{\beta}$.
Here we present a more general lower bound that applies to algorithms
outputting arbitrary hypotheses.  This includes algorithms that use weighted voting  (perhaps with  $L_1$ regularization).
In this section we
\begin{itemize} 
\item set 
the number of features $N$, number of relevant features $K$, and
sample size $m$ as a functions of $\gamma$ in such a way
that Corollary~\ref{c:inclusive.good} applies, and
\item prove a constant lower bound for these combinations
of values that holds for ``exclusive'' algorithms (defined below)
when $\gamma$ is small enough.
\end{itemize}
Thus, in this situation,  ``inclusive" algorithms relying on many irrelevant variables have error rates going to zero 
while every ``exclusive'' algorithm has an error rate bounded below by a constant.

The proofs in this section assume that all relevant variables are positively correlated with the class designation, so
each relevant variable agrees with the class designation with probability $1/2 +\gamma$.  
Although not essential for the results, this assumption simplifies the definitions and  notation\footnote{The 
assumption that each relevant variable agrees with the class label with probability $1/2+\gamma$ gives
a special case of the generative model described in Section~\ref{s:learning}, so the lower bounds proven here also
apply to that more general setting.}.
We also set $m = 2 \ln (32)/\gamma^2$.  
This satisfies the assumption of Theorem~\ref{t:beta.cgamma} when
$\beta \leq \gamma/2$ (see Inequality~(\ref{e:mThreshold})).

\begin{definition}
We say a classifier $f$ \underline{\em includes} a variable $x_i$ if there
is an input $(x_1,...,x_N)$ such that 
\[
f(x_1,...,x_{i-1},x_i,x_{i+1},...,x_N) \neq f(x_1,...,x_{i-1},1-x_i,x_{i+1},...,x_N).
\]
Let $V(f)$ be the set of variables included in $f$.
\end{definition}

For a training set $S$, we will refer to the classifier output
by algorithm $A$ on $S$ as $A(S)$.
Let $\cal R$ be the set of relevant variables.

\begin{definition}
We say that an algorithm $A$ is
\underline{\emph{$\lambda$-exclusive}}\footnote{The proceedings version
of this paper \citep{HL11} used a different definition of $\lambda$-exclusive.}
if for every positive $N$, $K$,
$\gamma$, and $m$, 
the expected fraction of the variables included in its hypothesis 
that are relevant is at least $\lambda$, i.e.~{
$ \displaystyle 
\E\left(\frac{|V(A(S)) \cap {\cal R}|}{|V(A(S))|}\right) \geq \lambda.
$
}
\end{definition}

Our main lower bound theorem is the following.
\begin{theorem}
\label{t:lower.lambda}
If 
\begin{align*}
K &= \frac{1}{\gamma^2} \EXP{\ln(1/ \gamma)^{1/3}} \\
N &= K \EXP{\ln(1/ \gamma)^{1/4}} \\
m &= \frac{2 \ln(32)}{\gamma^2}
\end{align*}
then for any constant $\lambda > 0$ and any $\lambda$-exclusive algorithm $A$, 
the error rate of $A$ is lower bounded by $\lambda/4-o(1)$ as $\gamma$ goes to 0.
\end{theorem}

Notice that this theorem provides a sharp contrast to
Corollary~\ref{c:inclusive.good}.
Corollary~\ref{c:inclusive.good}
shows that inclusive $A$ using models $\Mb$ for any  $0 \leq \beta \leq \gamma / 2$
have error rates that goes to zero super-polynomially 
fast (in $1/\gamma$) under the assumptions of Theorem~\ref{t:lower.lambda}.

The values of $K$ and $N$ in Theorem~\ref{t:lower.lambda} are
chosen to make the proof convenient, but other values would work.
For example, decreasing $K$ and/or increasing $N$ would make 
the lower bound part of Theorem~\ref{t:lower.lambda} easier to prove.  
There is some slack to do so while continuing to ensure that the upper bound
of Corollary~\ref{c:inclusive.good} goes to $0$.

As the correlation of variables with the label over the sample plays a
central role in our analysis, we will use the following definition.
\begin{definition} \label{d:empirical.edge}
If a variable agrees with the class label on $1/2 + \eta$ of the
training set then it has \emph{(empirical) edge} $\eta$.
\end{definition}

The proof of Theorem~\ref{t:lower.lambda} uses a critical value of
$\beta$, namely $\beta^* = \gamma \ln(N/K) / 10 \ln(32)$, with the property that
both:
\begin{eqnarray}
\frac{\E \left(| \Mcrit \cap \Rel |  \right)}{\E \left(  | \Mcrit | \right)} \rightarrow 0  \label{e:limzero} \\
\E \left( | \Mcrit \cap \Rel | \right) \in o(1/\gamma^2) \label{e:littleo}
\end{eqnarray}
as $\gamma \rightarrow 0$.

Intuitively, \eqref{e:limzero} means that any algorithm that  uses most of the variables having empirical edge at least $\bcrit$ cannot be $\lambda$-exclusive.
On the other hand, \eqref{e:littleo} implies that if the
algorithm restricts itself to variables with empirical edges greater
than $\bcrit$ then it does not include enough relevant variables to be accurate.
The proof must show that \emph{arbitrary} algorithms frequently  include either too many irrelevant variables 
to be $\lambda$-exclusive or too few relevant ones to be accurate.
See Figure~\ref{f:facts} for some useful facts
about $\gamma$, $m$, and $\bcrit$.

\begin{figure}
\begin{center}
\framebox{
\begin{minipage}[t]{1.0in}
\begin{align*}
b &= 2 \ln( 32)  \\[5pt]
m &= \frac{b}{ \gamma^2} = \frac{2 \ln(32)}{ \gamma^2 } \\[5pt]
\ \ \bcrit &= \frac{\gamma \ln(N/K) }{5b}  = \frac{\gamma \ln(1/\gamma)^{1/4}}{10 \ln (32) } \ \ \\
\end{align*}
\end{minipage}
}
\end{center}

\caption{Some useful facts relating $b$, $\gamma$, $m$ and $\bcrit$ under the assumptions of Theorem~\ref{t:lower.lambda}. }
\label{f:facts}
\end{figure}

To prove the lower bound, 
borrowing a technique from \cite{EHKV89}, we will assume that
the $K$ relevant variables are randomly selected from the $N$
variables, and lower bound the error with respect to this random
choice, along with the training and test data.  This will then imply
that, for each algorithm, there will be a choice of the $K$ relevant
variables giving the same lower bound with respect only to the random
choice of the training and test data.  We will always use relevant
variables that are positively associated with the class label,
agreeing with it with probability $1/2 + \gamma$.

\begin{proof}[of Theorem~\ref{t:lower.lambda}]
Fix any learning algorithm $A$, and let $A(S)$ be the hypothesis
produced by $A$ from sample $S$.  Let $n(S)$ be the number of
variables included in $A(S)$ and let $\beta(S)$ be the
$n(S)$'th largest empirical (w.r.t.\ $S$) edge of a variable.

Let $q_\gamma$ be the probability
 that $\beta(S) \geq \beta^* = \gamma \ln(N/K) / 10 \ln(32)$.
We will show in Section~\ref{s:exclusive} that if $A$ is $\lambda$-exclusive
then $\lambda \leq q_\gamma + o(1)$ (as $\gamma$ goes to 0).
We will also show in Section~\ref{s:error} that the
expected error of $A$ is at least $q_\gamma /4  - o(1)$ as $\gamma$ goes to 0.
Therefore any $\lambda$-exclusive algorithm $A$ has an expected error rate at least $\lambda/4-o(1)$ as $\gamma$ goes to 0.
\end{proof}

Before attacking the two parts of the proof alluded to above, we need
a subsection providing some basic results about relevant variables and
optimal algorithms.

\subsection{Relevant Variables and Good Hypotheses}

This section proves some useful facts about relevant variables and
good hypotheses.  The first lemma is a lower bound on the accuracy of
a model in terms of the number of relevant variables.

\begin{lemma}
\label{l:lower.relevant}
If $\gamma \in [0,1/5]$ then any classifier using $k$ relevant variables
has an error probability at least $\frac{1}{4} e^{-5 \gamma^2 k}$.
\end{lemma}
{\bf Proof}: The usual Naive Bayes calculation \citep[see][]{DHS00}
implies that the optimal classifier over a certain set $V$ of
variables is a majority vote over $V \cap {\cal R}$.  
Applying the lower tail bound (\ref{e:unfair}) then 
completes the proof. \qed

Our next lemma
shows that, given a sample, the probability that a variable is relevant
(positively correlated with the class label) is monotonically
increasing in its empirical edge.

\begin{lemma} \label{l:relevant.increasing}
For two variables $x_i$ and $x_j$, and any training set $S$
of $m$ examples,
\begin{itemize}
\item $\PR{x_i \text{ relevant} \mid S} > \PR{x_j \text{ relevant} \mid S}$
  if and only if the empirical edge of $x_i$ in $S$ is greater than
  the empirical edge of $x_j$ in $S$, and  
\item $\PR{x_i \text{ relevant} \mid S} = \PR{x_j \text{ relevant} \mid S}$
  if and only if the empirical edge of $x_i$ in $S$ is equal to
  the empirical edge of $x_j$ in $S$.  
\end{itemize}
\end{lemma}
\begin{proof}
Since the random choice of $\cal R$ does not effect that marginal
distribution over the labels,
we can generate $S$ by picking the labels for
all the examples first, then $\cal R$, and finally
the values of the variables on all the examples.
Thus if we can prove the lemma after conditioning on the values of the
class labels, then scaling all of the probabilities by $2^{-m}$
would complete the proof.  So, let us fix the
values of the class labels, and evaluate probabilities only
with respect to the random choice of the relevant variables $\cal R$, and 
the values of the variables.

Let
\[
\Delta = \PR{x_i \in \Rel| S} - \PR{x_j \in \Rel| S}.
\]
First, by subtracting off the probabilities that both variables are
relevant, we have
\[
\Delta = \PR{ x_i \in \Rel, x_j \not \in \Rel \big| S} 
               - \PR{ x_i \not\in \Rel, x_j \in \Rel \big| S}.
\]
\newcommand{\one}{\mathrm{ONE}}
Let $\one$ be the event that exactly one of $x_i$ or $x_j$ is relevant.
Then
\[
\Delta = (\PR{ x_i \in \Rel, x_j \not \in \Rel \big| S, \one} 
          - \PR{ x_i \not\in \Rel, x_j \in \Rel \big| S, \one})\PR{\one}.
\]
So $\Delta > 0$ if and only if
\[
\Delta' \eqdef \PR{ x_i \in \Rel, x_j \not \in \Rel \big| S, \one} 
            - \PR{ x_i \not\in \Rel, x_j \in \Rel \big| S, \one} > 0
\]
(and similarly for $\Delta = 0$ if and only if $\Delta' = 0$).  If
$\mathbb{Q}$ is the distribution obtained by conditioning on $\one$, then
\[
\Delta' = \QR{ x_i \in \Rel, x_j \not \in \Rel \big| S} 
            - \QR{ x_i \not\in \Rel, x_j \in \Rel \big| S}.
\]

Let $S_i$ be the values of variable $i$ in $S$, and define $S_j$
similarly for variable $j$.  Let $S'$ be the values of the other
variables.  Since we have already conditioned on the labels,
after also conditioning on $\one$ (i.e.,
under the distribution $\mathbb{Q}$), the pair $(S_i,S_j)$ is independent of
$S'$.  
For each $S_i$ we have $\PR{S_i \mid x_i \not\in \Rel} = \QR{S_i \mid x_i \not\in \Rel}$.
Furthermore, by symmetry, 
\[
\QR{ x_i \in \Rel, x_j \not \in \Rel \big| S'} = \QR{ x_i \not\in \Rel, x_j \in \Rel \big| S'} = \frac{1}{2}.
\]
Thus, by using Bayes' Rule on each term, we have
\begin{eqnarray*}
\Delta' 
  & = & \QR{ x_i \in \Rel, x_j \not \in \Rel \big| S_i, S_j, S'} 
              - \QR{ x_i \not\in \Rel, x_j \in \Rel \big| S_i, S_j, S'} \\
  & = & 
  \frac{\QR{ S_i, S_j\big| x_i \in \Rel, x_j \not \in \Rel, S'} 
              - \QR{ S_i, S_j \big| x_i \not\in \Rel, x_j \in \Rel, S'}}
       {2\QR{S_i,S_j | S'}} \\
  & = & 
  \frac{(1/2 + \gamma)^{m_i}(1/2 - \gamma)^{m - m_i}
        - (1/2 + \gamma)^{m_j}(1/2 - \gamma)^{m - m_j}}
       {2^{m+1} \QR{S_i,S_j}},
\end{eqnarray*}
where $m_i$ and $m_j$ are the numbers of times that variables
$x_i$ and $x_j$ agree with the label in sample $S$.
The proof concludes by observing that $\Delta'$ is positive 
exactly when $m_i > m_j$ and zero exactly when $m_i = m_j$.
\end{proof}

Because, in this lower bound proof, relevant variables are
always positively associated with the class label, 
we will use a variant of ${\cal M}_{\beta}$ 
which only considers positive features.
\begin{definition}
Let ${\cal V}_{\beta}$ be a vote over the variables
with empirical edge at least $\beta$.  
\end{definition}
When there is no chance of confusion, we will refer to the set
of variables in ${\cal V}_{\beta}$ also as ${\cal V}_{\beta}$
(rather than $V({\cal V}_{\beta})$).

We now establish lower bounds on the probability of variables being included in ${\cal V}_{\beta}$ 
(here $\beta$ can be a function of $\gamma$, but does not depend on the particular sample $S$).

\begin{lemma}
\label{l:irrelevant.lower}
If $\gamma \leq 1/8$  and $\beta \geq 0$ then
the probability that a given variable has empirical edge at least $\beta$ is at least
\[
\frac{1}{5} \exp\left(- 16  \beta^2 m \right) .
\]
If in addition $m \geq 1 / \beta^2$, then the probability that a given variable has empirical edge at least
$\beta$ is at least
\[
\frac{1}{7 \beta \sqrt{m}} 
         \exp\left(-2 \beta^2 m \right)
            - \frac{1}{\sqrt{m}}.
\]
\end{lemma}
{\bf Proof}:  
Since relevant variables agree with the class label with probability $1/2+\gamma$, the
probability that a relevant variable has empirical edge at least $\beta$ is lower bounded by 
the probability that an irrelevant variable has empirical edge at least $\beta$.
An irrelevant variable has empirical edge at least $\beta$ only when it agrees with the class on
$1/2+\beta$ of the sample.
Applying~\bndref{e:lower.fair}, this happens with probability at least 
\(
\frac{1}{5}
     \exp\left(- 16 \beta^2 m \right).
\)
The second part uses~\bndref{e:lower.tail} instead of (\ref{e:lower.fair}).
\qed

We now upper bound the probability of a relevant variable being included in
${\cal V}_{\beta}$, again for $\beta$ that does not depend on $S$.

\begin{lemma}
\label{l:relevant.upper}
If $\beta \geq \gamma$, the probability that a given relevant variable has
empirical edge at least $\beta$ is at most $e^{-2(\beta - \gamma)^2 m}$.
\end{lemma}
{\bf Proof}:  
Use~(\ref{e:hoeffding}) to bound the probability that a relevant feature agrees
with
the class label $\beta - \gamma$ more often than its expected
fraction of times.
\qed

\subsection{Bounding $\lambda$-Exclusiveness}
\label{s:exclusive}

Recall that $n(S)$ is the number of variables used by $A(S)$, 
and $\beta(S)$ is the edge of the variable whose rank, when
the variables are ordered by their empirical edges, is $n(S)$.
We will show that: if $A(S)$ is
$\lambda$-exclusive, then there is reasonable probability that $\beta(S)$ is 
at least the critical value $\beta^* = \gamma \ln (N/K) / 5b$.
Specifically, if $A$ is $\lambda$-exclusive, then,
for any small enough $\gamma$, we have $\displaystyle
\PR{\beta(S) \geq \beta^* } > \lambda/2$.

\newcommand{\Irrel}{ { \cal I } }
\newcommand{\Mbs}{ {\cal M}_{\beta(S)} }
\newcommand{\Vcrit}{  {\cal V}_{\bcrit} }

Suppose, given the training set $S$, the variables are sorted in
decreasing order of empirical edge (breaking ties arbitrarily, say using
the variable index).  
Let ${\cal V}_{S,k}$ consist of the first
$k$ variables in this sorted order, the ``top $k$'' variables.

Since for each sample $S$ and each variable $x_i$, the probability
$\PR{x_i \text{ relevant} \big| S}$ decreases as the empirical edge of
$x_i$ decreases (Lemma~\ref{l:relevant.increasing}), the expectation
$\displaystyle \E \left( \frac{|{\cal V}_{S,k} \cap \Rel |}{| {\cal V}_{S,k} |} \;\bigg|\; S \right) $ 
is non-increasing with $k$.

Furthermore, Lemma~\ref{l:relevant.increasing} also implies that for each sample
$S$, we have
\[
\E \left(\frac{ | V(A(S)) \cap \Rel |}{|V(A(S))|}  \;\bigg|\; S \right) 
\leq 
\E \left( \frac{| {\cal V}_{S,n(S)} \cap \Rel |} 
               { | {\cal V}_{S,n(S)} |}  
\;\bigg|\; S \right) .
\]
Therefore, by averaging over samples, for each $\gamma$ we have
\begin{equation}
\label{e:by.M}
\E \left( \frac{ | V(A(S)) \cap \Rel |  }{ |V(A(S))|  } \right)
	\leq 
\E \left( \frac{| {\cal V}_{S,n(S)} \cap \Rel |} 
               { | {\cal V}_{S,n(S)} |}\right).
\end{equation}

Note that the numerators in the expectations are never greater than the denominators.
We will next give upper bounds 
on $| \Vcrit \cap \Rel |$ 
and lower bounds on $| \Vcrit |$ that each 
hold with probability $1 - \gamma$.  

The next step is a high-confidence upper bound on $| \Vcrit \cap \Rel
|$.  From Lemma~\ref{l:relevant.upper}, the probability that a
particular relevant variable is in $\Vcrit$ is at most (recall that $m
= b/\gamma^2$)
\begin{align*}
e^{-2(\bcrit - \gamma)^2 m}  
&=   e^{-2 b (\bcrit / \gamma - 1)^2} \\ 
&=  \exp \left(-2 b \left( \frac{\ln(1/\gamma)^{1/4}}{5b} - 1 \right)^2 \right)   \\ 
&=  \exp \left(\frac{-2\ln(1/\gamma)^{1/2}}{25b} + \frac{ 4 \ln(1/\gamma)^{1/4} }{ 5} - 2b \right)   \\
& < \exp \left(\frac{-2\ln(1/\gamma)^{1/2}}{25b}  + \frac{4 \ln(1/\gamma)^{1/4}}{5} \right).
\end{align*}

\newcommand{\prel}{p_{\text{rel}}}
\newcommand{\pirrel}{p_{\text{irrel}}}

Let $\prel = \exp \left(\frac{-2\ln(1/\gamma)^{1/2}}{25b} + \frac{4
  \ln(1/\gamma)^{1/4}}{5} \right)$ be this upper bound, and note that
$\prel$ drops to 0 as $\gamma$ goes to 0, but a rate slower than
$\gamma^\epsilon$ for any $\epsilon$.  The number of relevant
variables in $\Vcrit$ has a binomial distribution with parameters $K$
and $p$ where $p < \prel$.  The standard deviation of this
distribution is 
\begin{equation}
\label{e:sigma.bound}
\sigma = \sqrt{K p (1-p)} < \sqrt{Kp} < \frac{\exp \left( \frac{\ln(1/\gamma)^{1/3} }{ 2 }\right) \sqrt{\prel} } {\gamma}  .
\end{equation}

Using the Chebyshev bound, 
\begin{equation}  \label{e:chebyshev}
\PR{| X - \E(X)| > a \sigma} \leq \frac{1}{a^2}
\end{equation}
with $a = 1 / \sqrt{\gamma}$
gives that
\begin{eqnarray}
\nonumber
\PR{ | \Vcrit  \cap \Rel | - Kp > \frac{\sigma}{\sqrt{\gamma} } }&\leq& \gamma \\
\label{e:two.term}
\PR{ | \Vcrit  \cap \Rel |  > K \prel + \frac{\sigma}{\sqrt{\gamma} } }&\leq& \gamma.
\end{eqnarray}

Since $\sigma < \sqrt{Kp} < \sqrt{K \prel}$ by~(\ref{e:sigma.bound}), we have $\sigma \sqrt{K\prel} < {K \prel}$.
 Substituting the values of $K$ and $\prel$ into the square-root yields 
\begin{eqnarray*}
K \prel & > & \sigma \times 
 \frac{\exp \left( \frac{\ln(1/\gamma)^{1/3} }{ 2 }\right) \exp \left(\frac{-\ln(1/\gamma)^{1/2}}{25b} + \frac{2
  \ln(1/\gamma)^{1/4}}{5} \right)} {\gamma} \\
    & > & \sigma/\sqrt{\gamma},
\end{eqnarray*}
for small enough $\gamma$.  Combining with (\ref{e:two.term}), we
get that
\begin{equation} \label{e:mbr.upper}
\PR{ | \Vcrit  \cap \Rel |  > 2K \prel } \leq \gamma
\end{equation}
holds for small enough $\gamma$.

Using similar reasoning, we now obtain a lower bound on 
the expected number of variables in $\Vcrit$.
Lemma~\ref{l:irrelevant.lower} shows that, for each variable, the probability of the variable having empirical edge $\bcrit$
is at least 
\begin{align*}
\frac{1}{7 \bcrit \sqrt{m}} 
         \exp\left(-2 \bcrit^2 m \right)
            - \frac{1}{\sqrt{m}}
&= \frac{5 \sqrt{b}}{7 \ln(1/\gamma)^{1/4} } 
         \exp\left(-2 \frac{\ln(1 / \gamma)^{1/2} }{25b } \right)
            - \frac{\gamma}{\sqrt{b}} \\
&> \frac{ \sqrt{b}}{2 \ln(1/\gamma)^{1/4} } 
         \exp\left(-2 \frac{\ln(1 / \gamma)^{1/2} }{25b } \right)  
\end{align*}
for sufficiently small $\gamma$.
Since the empirical edges of different variables are independent,
the probability that at least $n$ variables have empirical edge at least $\bcrit$ is lower bounded by
the probability of at least $n$ successes from the binomial distribution with parameters $N$ and $\pirrel$ where
\[
\pirrel = \frac{ \sqrt{b}}{2 \ln(1/\gamma)^{1/4} } 
         \exp\left(-2 \frac{\ln(1 / \gamma)^{1/2} }{25b } \right).
\]
If, now, we define $\sigma$ to be the standard deviation
of this binomial distribution, then,
like before, $\sigma  = \sqrt{N \pirrel (1-\pirrel)} < \sqrt{N\pirrel} $,
and
\begin{eqnarray*}
N \pirrel / 2 & > & \sigma \sqrt{N \pirrel}/2 \\
 & = & \frac{\sigma}{2} \times \frac{\exp((1/2)(\ln(1/\gamma)^{1/4} + \ln(1/\gamma)^{1/3}))}{\gamma}
   \times 
   \frac{ b^{1/4}}{\sqrt{2} \ln(1/\gamma)^{1/8} } 
         \exp\left(-\frac{\ln(1 / \gamma)^{1/2} }{25b } \right),
\end{eqnarray*}
so that, for small enough $\gamma$, $N \pirrel / 2 > \sigma / \sqrt{\gamma}$.
Therefore applying the Chebyshev bound~(\ref{e:chebyshev}) 
with $a= 1/\sqrt{\gamma}$ gives (for sufficiently small $\gamma$)
\begin{equation}  \label{e:mb.lower}
\PR{ | \Vcrit | < \frac{N \pirrel}{2} } 
\leq \PR{ | \Vcrit | < N \pirrel - \sigma/\sqrt{\gamma} } 
< \gamma.
\end{equation} 

Recall that 
\[
q_\gamma = \PR{\beta(S) \geq \beta^*}  = \PR{n(S) \leq \vert \Vcrit \vert}.
\]
If $A$ is $\lambda$-exclusive then, using (\ref{e:by.M}), we have
\begin{align*}
\lambda \leq \E \left( \frac{ | V(A(S)) \cap \Rel | ) }{ |V(A(S))|  } \right)
	& \leq 
\E \left( \frac{| {\cal V}_{S,n(S)} \cap \Rel |} 
               { | {\cal V}_{S,n(S)} |}\right) \\
	&\leq  (1 - q_\gamma ) \E \left( \frac{| {\cal V}_{S,n(S)} \cap \Rel |} 
               { | {\cal V}_{S,n(S)} |} \;\bigg|\; |\Vcrit| < n(S) \right)  + q_\gamma  \\
	&\leq  (1 - q_\gamma ) \E \left( \frac{ | \Vcrit \cap \Rel |   }{  | \Vcrit |}  \;\bigg|\; |\Vcrit| < n(S) \right)  + q_\gamma  \\
	&\leq  (1 - q_\gamma ) \left( \frac{2K \prel} { N \pirrel/2} + 2 \gamma \right)  + q_\gamma 
\end{align*}
where we use the upper and lower bounds from
Equations~(\ref{e:mbr.upper}) and~(\ref{e:mb.lower}) that each hold
with probability $1 - \gamma$. 
Note that the
ratio

\begin{align*}
\frac{2K \prel} {N \pirrel/2}
	&\leq \frac{ \displaystyle 2 \frac{e^{\ln(1/\gamma)^{1/3} }}{\gamma^2}  \exp \left(\frac{-2\ln(1/\gamma)^{1/2}}{25b}  + \frac{4 \ln(1/\gamma)^{1/4}}{5} \right)  }
			{ \displaystyle \frac{e^{ \ln(1/\gamma)^{1/3}} e^{\ln(1 / \gamma)^{1/4}} }{ 4 \gamma^2}   
					\frac{ \sqrt{b}}{ \ln(1/\gamma)^{1/4} }  \exp\left(-2 \frac{\ln(1 / \gamma)^{1/2} }{25b } \right) }\\
	&= \frac{ \displaystyle  8   \ln(1/\gamma)^{1/4}  \exp \left(\frac{-2\ln(1/\gamma)^{1/2}}{25b}  + \frac{4 \ln(1/\gamma)^{1/4}}{5} \right) }
			{ \displaystyle  \sqrt{b} \: e^{\ln(1 / \gamma)^{1/4} }
					    \exp\left(-2 \frac{\ln(1 / \gamma)^{1/2} }{25b } \right) } \\
	&=  \frac{ \displaystyle  8   \ln(1/\gamma)^{1/4} \exp \left( \frac{- \ln(1/\gamma)^{1/4}}{5} \right) }
			{ \displaystyle  \sqrt{b}  	  } 				     
\end{align*}
which goes to 0 as $\gamma$ goes to 0.
Therefore,  
\[
\lambda \leq \E \left(\frac{ | V(A(S)) \cap \Rel | ) }{ |V(A(S))|  } \right) 
        \leq q_\gamma + o(1)
\]
which implies that,
\[
q_\gamma = \displaystyle \PR{\beta(S) \geq \beta^*} \geq \lambda -o(1)
\]
as $\gamma$ goes to 0.

\subsection{Large Error}
\label{s:error}

Call a variable \emph{good} if it is relevant and its 
empirical edge is at least $\bcrit$ in the sample.
Let $p$ be the probability that a relevant variable is good.
Thus the number of good variables is binomially distributed with parameters $K$ and $p$.
We have that the expected number of good variables is $pK$ and the variance is $Kp(1-p) < Kp$.
By Chebyshev's inequality, we have
\begin{equation}
\label{e:a.bound}
\PR{ \text{\# good vars} \geq Kp + a \sqrt{Kp} }  \leq \PR{ \text{\# good vars} \geq Kp + a \sqrt{Kp(1-p)}  } \leq  \frac{1}{a^2}, 
\end{equation}
and setting  $a= \sqrt{Kp}$,  this gives
\begin{equation}
\PR{ \text{\# good vars} \geq 2Kp } \leq \frac{1}{Kp}.
\end{equation}

By Lemma~\ref{l:relevant.upper}, $\displaystyle Kp \leq K e^{-2(\bcrit-\gamma)^2 m} = K e^{-2  b \left(\ln(1/\gamma)^{1/4}  /5b -1 \right)^2} $, so
\begin{align*}
\ln(Kp) &\leq \ln K -2b \left( \frac{ \ln(1 / \gamma)^{1/2} }{ 25 b^2} - \frac{2 \ln(1/\gamma)^{1/4}}{5b} +1 \right) \\
\ln(Kp) &\leq 2 \ln(1/\gamma) + \ln(1/\gamma)^{1/3} -2b \left( \frac{ \ln(1 / \gamma)^{1/2} }{ 25 b^2} - \frac{2 \ln(1/\gamma)^{1/4}}{5b} +1 \right).
\end{align*}
So for small enough $\gamma$, 
\[ 
\ln( Kp)  \leq 2 \ln (1/\gamma) - \frac{\ln(1/\gamma)^{1/2}}{25b}
\]
and thus $Kp \in o( 1/ \gamma^2)$.

So if $Kp > 1/\gamma$, then with probability at least $1-\gamma$, there
are less than $2Kp \in o(1/\gamma^2)$ good variables.
On the other hand, if $Kp < 1/ \gamma$, then, setting $a = \sqrt{1/\gamma}$
in bound~(\ref{e:a.bound})
gives that the probability  that there are more than $2 / \gamma$ good 
variables is at most  $\gamma$. 
So in either case the probability that there are more than 
$\frac{2}{ \gamma^2 } \EXP{- \ln( 1/\gamma)^{1/2} / 25b}$
good variables is at most $\gamma$
(for small enough $\gamma$).

So if $\displaystyle \PR{\beta(S) \geq \bcrit }
\geq q_\gamma$, then with probability at least $q_\gamma - \gamma$
algorithm $A$ is using a hypothesis with 
at most 
$\frac{2}{ \gamma^2 } \EXP{- \ln( 1/\gamma)^{1/2} / 25b}$
relevant variables.  Applying Lemma~\ref{l:lower.relevant} yields the
following lower bound on the probability of error:
\begin{equation}
\label{e:final.lower}
(q_\gamma - \gamma) \frac{1}{4} \EXP{ -10 \EXP{- \ln( 1/\gamma)^{1/2} / 25b}}.
\end{equation}
Since the limit of (\ref{e:final.lower}) for small $\gamma$ is
$q_\gamma /4$, this completes the proof of Theorem~\ref{t:lower.lambda}.

\section{Relaxations of some assumptions}

To keep the analysis clean, and facilitate the interpretation of the
results, we have analyzed an idealized model.  In this section,
we briefly consider the consequences of some relaxations of our
assumptions.
       
\subsection{Conditionally dependent variables}
\label{s:dependent}

Theorem~\ref{t:no.learning.misleading} can be generalized to the case
in which there is limited dependence among the variables, after
conditioning on the class designation, in a variety of
ways.  For example, suppose that there is a degree-$r$ graph $G$ whose nodes
are variables, and such that, conditioned on the label, each variable
is independent of all variables not connected to it by an edge in $G$.
Assume that $k$ variables agree with the label with probability $1/2 +
\gamma$, and the $n-k$ agree with the label with probability $1/2$.
Let us say that a source like this {\em has $r$-local dependence}.
Then applying a
Chernoff-Hoeffding bound for such sets of random variables due to
\citet{Pem01}, if $r \leq n/2$, one gets a bound of
$c (r+1) \exp\left(\frac{-2\gamma^2 k^2}{n (r + 1)}\right)$
the probability of error.

\subsection{Variables with different strengths}
\label{s:different}

We have previously assumed that all relevant variables are equally
strongly associated with the class label.  
Our analysis is easily
generalized to the situation when the strengths of associations
fall in an interval $[\gmin, \gmax]$. 
Thus relevant variables agree with the class label with probability at least $1/2 + \gmin$ and
misleading variables agree with the class label with probability at least $1/2 - \gmax$.
Although a sophisticated analysis would take each variable's degree of association
into account,  it is possible to leverage our previous analysis with a simpler approach.
Using the $1/2 + \gmin$ and $1/2 - \gmax$ underestimates on the probability that relevant variables and
misleading variables agree with the class label leads to an analog of 
Theorem~\ref{t:no.learning.misleading}.
This analog says that models voting $n$ variables, $k$ of which are relevant and $\ell$ of which are misleading,
have error probabilities bounded by 
\[
\displaystyle \exp\left(\frac{-2 [\gmin k  - \gmax \ell]_+^2 }{n} \right).
\]

We can also use the upper and lower bounds on association to get high-confidence bounds 
(like those of Lemmas~\ref{l:few.misleading} and \ref{l:many.relevant})  
on the numbers of relevant and misleading features 
in models $\Mb$.
This leads to an analog of Theorem~\ref{t:learning.beta} bounding the expected error rate of $\Mb$ by 
\[
(1 + o(1)) 
  \exp \left( \frac{ - 2 \gmin^2  K \left[1 - 4 (1+ \gmax / \gmin) e^{-2 (\gmin - \beta)^2 m} 
                                           -  \gmin \right]_+^2}
                   {1+ \frac{8 N}{K} e^{-2 \beta^2 m} +  \gmin} \right)
\]
when  $0 \leq \beta \leq \gamma$ and $\gmax \in o(1)$.
Note that $\gamma$ in Theorem~\ref{t:learning.beta} is replaced by 
$\gmin$ here, and $\gmax$ only appears in the $4(1 + \gmax / \gmin)$ factor (which replaces an ``8'' in the original theorem).

Continuing to mimic our previous analysis gives analogs to
Theorem~\ref{t:beta.cgamma} and Corollary~\ref{c:inclusive.good}.
These analogs imply that if $\gmax / \gmin$ is bounded then algorithms
using small $\beta$ perform well in the same limiting situations used
in Section~\ref{s:lower} to bound the effectiveness of exclusive
algorithms.

A more sophisticated analysis keeping better track of the degree of association 
between relevant variables and the class label may produce better bounds.
In addition, if the variables have 
varying strengths then it makes sense to consider classifiers that
assign different voting weights to the variables based on their estimated
strength of association with the class label.
An analysis that  takes account of these
issues is a potentially interesting subject for further research.  

\section{Conclusions}

We analyzed learning when there are few examples, a small fraction of
the variables are relevant, and the relevant variables are only weakly
correlated with the class label.  In this situation, algorithms that
produce hypotheses consisting predominately of irrelevant variables
can be highly accurate (with error rates going to 0).  Furthermore,
this inclusion of many irrelevant variables is essential.  Any
algorithm limiting the expected fraction of irrelevant variables in
its hypotheses has an error rate bounded below by a constant.  This is
in stark contrast with many feature selection heuristics that limit
the number of features to a small multiple of the number of examples,
or that limit the classifier to use variables that pass stringent
statistical tests of association with the class label.

These results have two implications on the practice of machine
learning.  First, they show that the engineering practice of producing
models that include enormous numbers of variables is sometimes
justified.  Second, they run counter to the intuitively
appealing view that accurate class prediction ``validates''
the variables used by the predictor.

\section*{Acknowledgements}

We thank Aravind Srinivasan for his help.

\appendix

\section{Appendices}

\subsection{Proof of~(\protect\ref{e:chernoff}) and~(\protect\ref{e:leq4})}
\label{a:chernoff}

Equation~4.1 from \citep{MR95} is
\begin{equation}  \label{e:chernoff.MR}
\PR{U > (1+\eta)\E(U)}
   < \left(\frac{e^\eta}{(1+\eta)^{1+\eta}} \right)^{\E(U)}
\end{equation}
and holds for independent 0-1 valued $U_i$'s each with (possibly different) probabilities $p_i=P(U_i=1)$
where $0 < p_i < 1$ and $\eta > 0$.
Taking the logarithm of the RHS, we get
\begin{eqnarray}
\ln(\mbox{RHS})  &=& \E(U) \left(\eta  - (1+\eta) \ln(1+\eta) \right)  \label{e:log.rhs} \\
&<& \E(U) \left(\eta+1  - (1+\eta) \ln(1+\eta) \right)    \nonumber \\
&=&  -\E(U) (\eta+1)  (\ln (1+\eta) -1 ),  \nonumber
\end{eqnarray}
which implies (\ref{e:chernoff}).  
From~(\ref{e:chernoff.MR}), when $0 \leq \eta \leq 4$ (since $\eta- (1+\eta) \ln(1+\eta) < -\eta^2 / 4$ there),
$
\PR{U>(1+\eta)\E(U)}
  < \exp\left(-\eta^2 \E(U)/4 \right)
$
showing~(\ref{e:leq4}).

\subsection{Proof of~(\protect\ref{e:highconf})}  
\label{a:highconf}
Using (\ref{e:chernoff}) with $\eta = 3+3 \ln (1/\delta)/\E(U)$ gives
\begin{align*}
\PR{U> 4\E(U) + 3 \ln(1/\delta) } 
& < \exp\left(-(4 \E(U)+ 3 \ln \delta) \ln \left(\frac{4 + 3 \ln(1/\delta)/\E(U) }{e}\right) \right) \\   
& < \exp\left(-( 3 \ln(1/\delta) \ln \left(\frac{4}{e}\right) \right)   \\
& < \exp\left(- \ln(1/\delta) \right) = \delta
\end{align*}
using the fact that $\ln(4/e) \approx 0.38 > 1/3$.

\subsection{Proof of~(\protect\ref{e:lower.tail})}
\label{a:lower.tail}

The following is a straightforward consequence of the
Berry-Esseen inequality.
\begin{lemma}[{\citep[see][Theorem 11.1]{Das08}}]
\label{l:be}
Under the assumptions of Section~\ref{s:tools} with each $\PR{U_i=1} = 1/2$, let: 
\begin{align*}
T_i &= 2 (U_i-1/2),  \\
T &= \sqrt{\frac{1}{\ell}} \, \sum_{i=1}^{\ell} T_i, \text{ and} \\
Z &\text{  be a standard normal random variable. }
\end{align*}
Then for all $\eta$, we have
\( \displaystyle
\bigl\lvert \PR{ T > \eta} - \PR{ Z > \eta } \bigr\rvert \leq \frac{1}{\sqrt{\ell}}.
\)
\end{lemma}

\begin{lemma}[{\citep[][Chapter VII, section 1]{Fel68}}]
\label{l:gaussian.lower}
If $Z$ is a standard normal random variable and $x > 0$, then
\[
 \frac{1}{\sqrt{2 \pi}} 
        \left(\frac{1}{x} - \frac{1}{x^3}\right) e^{-x^2/2}
< \PR{Z > x}   <
 \frac{1}{\sqrt{2 \pi}} 
        \left( \frac{1}{x} \right) e^{-x^2/2}.
\]  

\end{lemma}

Now, to prove (\ref{e:lower.tail}),
let $M = \frac{1}{\ell} \sum_{i=1}^{\ell} (U_i - \frac{1}{2})$ and
let $Z$ be a standard normal random variable.  
Then Lemma~\ref{l:be} implies that, for all $\kappa$
\[
\left| \PR{ 2 \sqrt{\ell} M >  \kappa }
- \PR{Z > \kappa} \right|
 \leq \frac{1}{\sqrt{\ell}}.
\]
Using 
$\kappa = 2 \eta \sqrt{\ell}$, 
\begin{equation}
\label{e:bin.by.gaussian}
\PR{M > \eta} 
  \geq \PR{Z > 2 \eta \sqrt{\ell} }
             - \frac{1}{\sqrt{\ell}}.
\end{equation}
Applying Lemma~\ref{l:gaussian.lower}, we get
\[
\PR{Z > 2 \eta \sqrt{\ell} }
  \;\; \geq \;\; \frac{1}{\sqrt{2 \pi}} 
     \left(\frac{1}{2 \eta \sqrt{\ell}}
           - \left(\frac{1}{2 \eta \sqrt{\ell}}\right)^3 
                \right)
       e^{-2 \eta^2 \ell}.
\]
Since $\ell \geq 1/\eta^2$, we get
\begin{align*}
\PR{ Z > 2 \eta \sqrt{\ell} }
  \;\; &\geq \;\; \frac{1}{\sqrt{2 \pi}}
     \left(\frac{1}{2} - \frac{1}{8} \right)
     \frac{1}{\eta \sqrt{\ell}}
       e^{-2 \eta^2 \ell} \\
&\geq \;\;   \frac{1}{7 \eta \sqrt{\ell}}   e^{-2 \eta^2 \ell}.
\end{align*}
Combining with (\ref{e:bin.by.gaussian}) completes the proof
of~(\ref{e:lower.tail}). \qed

\subsection{Proof of (\protect\ref{e:lower.fair})}
\label{a:lower.fair}

We follow the proof of Proposition 7.3.2 in \citep[][Page 46]{MV11}.

\begin{lemma}
For $n$ even, let $U_1, \ldots, U_n$ be i.i.d.~RVs with $\PR{U_1 = 0} = \PR{U_1 = 1} = 1/2$
and $U = \sum_{i=1}^n$. 
Then for integer $t \in [0, \frac{n}{8}]$,
\[
\PR{U \geq \frac{n}{2} + t} \geq \frac{1}{5} e^{-16 t^2 / n} .
\]
\end{lemma}

\begin{proof}
Let integer $m= n/2$.
\begin{align}
\PR{U \geq m+t} &= 2^{-2m} \sum_{j=t}^m { 2m \choose m+j} \\
&\geq 2^{-2m} \sum_{j=t}^{2t-1} { 2m \choose m+j} \\
&= 2^{-2m} \sum_{j=t}^{2t-1} { 2m \choose m} \cdot \frac{m}{m+j} \cdot \frac{m-1}{m+j-1} \cdots \frac{m-j+1}{m+1} \\
&\geq \frac{1}{2 \sqrt{m}}   \sum_{j=t}^{2t-1} \prod_{i=1}^{j} \left(  1- \frac{j}{m+1}  \right) 
				\hspace{0.3in} \mbox{using ${2m \choose m} \geq  2^{2m} / 2 \sqrt{m}$ }\\
&\geq \frac{t}{2 \sqrt{m}}  \left(  1- \frac{2t}{m}   \right)^{2t} \\
&\geq \frac{t}{2 \sqrt{m}}  e^{- 8t^2 / m} 
				\hspace{0.3in} \mbox{ since $1-x \geq e^{-2x}$ for $0 \leq x \leq 1/2$.}
\end{align}
For $t \geq \frac{1}{2} \sqrt{m}$, the last expression is at least $\frac{1}{4}  e^{-16 t^2 / n}$. 

Note that
$\PR{U = m } = 2^{-2m} {2m \choose m} \leq 1 / \sqrt{\pi m} $.
Thus for $0 \leq t < \frac{1}{2} \sqrt{m}$, we have
\begin{align}
\PR{U \geq m+t} &\geq \frac{1}{2} - t \PR{ U = m } \\
& \geq \frac{1}{2} - \frac{1}{2}\sqrt{m} \frac{1}{\sqrt{\pi m}}   \\
& \geq \frac{1}{2} - \frac{1}{2 \sqrt{\pi}} \approx 0.218 \geq \frac{1}{5} \geq \frac{1}{5} e^{-16^2/ n}
\end{align}
Thus the  bound $\frac{1}{5} e^{-16 t^2/ n}$ holds for all $0 \leq t \leq m/4$. 
\end{proof}

\subsection{Proof of (\protect\ref{e:unfair})}
\label{a:unfair}

The proof of~(\ref{e:unfair}) uses the next two lemmas and 
and follows the proof of Lemma 5.1 in~\citep{AB99}.

\begin{lemma}[Slud's Inequality, \citep{Slu77}]
\label{l:slud}
Let $B$ be a binomial $(\ell,p)$ random variable with $p \leq 1/2$.
Then for $\ell(1-p) \geq j \geq \ell p$,
\[
\PR{B \geq j} \geq \PR{ Z \geq \frac{j-\ell p}{\sqrt{\ell p(1-p)} } } 
\]
where $Z$ is a standard normal random variable.
\end{lemma}

\begin{lemma}[{\citep[see][Appendix 1]{AB99}}]
\label{l:sqrtNormBound}
If $Z$ is a standard normal and $x>0$ then 
\[
\PR{ Z \geq x } \geq \frac{1}{2} \left( 1- \sqrt{1-e^{-x^2}}   \right).
\]
\end{lemma}

Recall that in~(\ref{e:unfair}) $U$ the sum of the $\ell$ i.i.d.~boolean random variables, each of which is $1$ with probability
$\frac{1}{2} + \eta$.  Let $B$ be a random variable with the binomial $(\ell, \frac{1}{2}-\eta)$ distribution.

\begin{align*}
\PR{ \frac{1}{\ell} U < 1/2 } &= \PR{ B \geq \ell / 2 } \\
&\geq \PR{ N \geq \frac{\ell/2-\ell(1/2 - \eta)}{\sqrt{\ell (1/2 + \eta) (1/2 - \eta)} } }  & \text{Slud's Inequality} \\
&= \PR{ N \geq \frac{2 \eta \sqrt{\ell} }{\sqrt{(1 - 4\eta^2)} } }  \\
& \geq \frac{1}{2} \left( 1- \sqrt{1- \exp \left(   - \frac{4 \eta^2 \ell }{1 - 4\eta^2 }  \right) }   \right) \\
& \geq \frac{1}{4} \exp \left(  -  \frac{4 \eta^2 \ell }{1 - 4\eta^2 }  \right)   & \text{since $1-\sqrt{1-x} > x/2$ } \\
& \geq \frac{1}{4} \exp \left(  -  5 \eta^2 \ell  \right)   & \text{when $\eta \leq 1/5$}  
\end{align*}
completing the proof of~(\ref{e:unfair}).

\end{document}